\documentclass[11pt]{article}

\usepackage[final]{acl}

\usepackage{times}
\usepackage{latexsym}

\usepackage[T1]{fontenc}

\usepackage[utf8]{inputenc}

\usepackage{microtype}

\usepackage{inconsolata}

\usepackage{graphicx}
\usepackage{booktabs}   
\usepackage{tabularx}   
\usepackage{enumitem}
\usepackage{amssymb}
\usepackage{multirow}
\usepackage{soul}
\usepackage{amsthm}
\newtheorem{lemma}{Lemma}
\usepackage[most]{tcolorbox}

\title{LoopGuard: Breaking Self-Reinforcing Attention Loops \\via Dynamic KV Cache Intervention}

\author{
 \textbf{Dongjie Xu\textsuperscript{1}},
 \textbf{Hao Wu\textsuperscript{2}},
 \textbf{Weijie Shi\textsuperscript{3}},
 \textbf{Yue Cui\textsuperscript{4}},
 \textbf{Yuanjun Liu\textsuperscript{1}}, \\
 \textbf{Jiawei Li\textsuperscript{3}},
 \textbf{Haolun Ma\textsuperscript{3}},
 \textbf{An Liu\textsuperscript{1}},
 \textbf{Jia Zhu\textsuperscript{5}},
 \textbf{Jiajie Xu\textsuperscript{1}}
 \\
 \\
 \textsuperscript{1}Soochow University,\\
 \textsuperscript{2}Tongji University,\\
 \textsuperscript{3}The Hong Kong University of Science and Technology,\\
 \textsuperscript{4}Alibaba Group,\\
 \textsuperscript{5}Zhejiang Normal University
}

\begin{document}
\maketitle
\begin{abstract}

Through systematic experiments on long-context generation, we observe a damaging failure mode in which decoding can collapse into persistent repetition loops. We find that this degeneration is driven by collapsed attention patterns, where a subset of heads locks onto a narrow suffix of the history, and is further stabilized by inference-time KV cache reuse. Crucially, since many existing KV cache policies rely on attention-based importance, this collapse can produce spuriously high scores for repetitive tokens, causing cache management to inadvertently amplify repetition. To study this phenomenon in a controlled and reproducible manner, we introduce LoopBench, a benchmark with explicit loop-inducing conditions and loop-oriented metrics that quantify repetition severity and generation instability beyond downstream task scores. Building on these insights, we propose LoopGuard, a lightweight, plug-in KV cache guard that detects loop onset online and disrupts the feedback cycle by pruning repetitive tail spans under a fixed cache budget. Experiments on LoopBench show that LoopGuard reduces loop incidence by over 90 percentage points, while restoring output diversity and reducing token waste. 
\end{abstract}

\begin{figure*}[t]
    \centering
    \includegraphics[width=\linewidth]{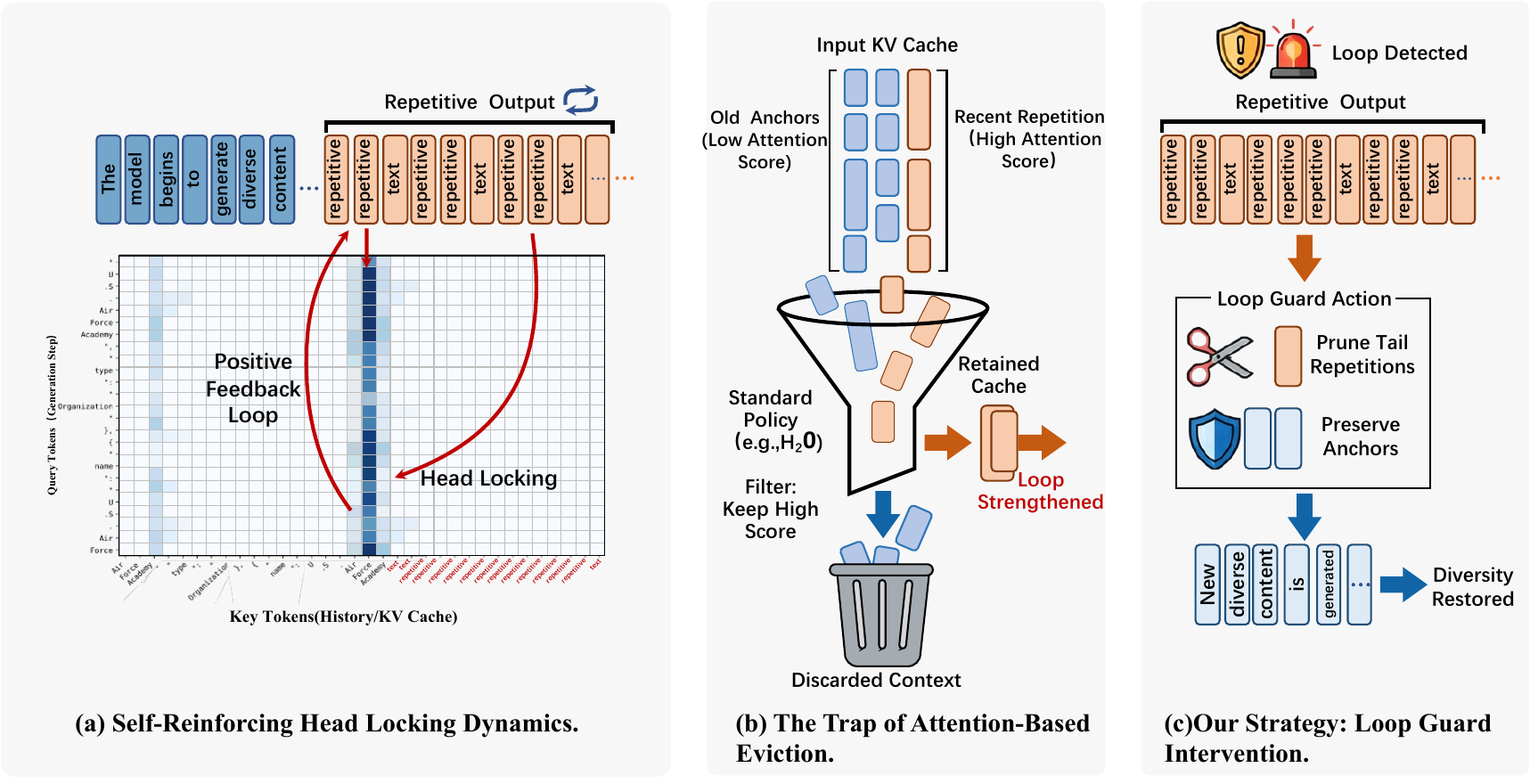}
    \caption{Illustration of self-reinforcing repetition loops and our mitigation strategy. (a) Attention heads lock onto narrow repetitive spans, creating a positive feedback loop. (b) Existing attention-based eviction policies (e.g., H2O) inadvertently strengthen this loop by preserving high-score repetitive tokens while discarding diverse context. (c) LoopGuard detects loop onset and prunes the repetitive tail to break the feedback cycle and restore diversity.}
    \label{fig:motivation}
\end{figure*}

\section{Introduction}
Large language models (LLMs) are increasingly deployed in long-context generation settings, powering workflows such as long-form summarization, structured information extraction, and tool-augmented agents~\cite{bai2024longbench}. In these applications, reliability is as critical as raw capability: failures can propagate to downstream systems and incur substantial computational cost~\cite{yao2022react}. Yet, despite advances in model scaling and long-context training, LLMs remain vulnerable to generation degeneration that can abruptly derail outputs, waste tokens, and inflate latency~\cite{holtzman2019curious}. Understanding and mitigating such reliability failures is therefore essential for dependable long-context generation.

In this work, we study persistent repetition loops, a damaging failure mode in which long-context decoding collapses into self-sustaining repetition and fails to recover contextual diversity~\cite{li2023repetition}. At inference time, long-context generation crucially relies on the key–value (KV) cache, which stores past keys and values to enable efficient attention over extended histories. This mechanism underpins many practical applications that require prolonged decoding, such as long-form summarization, structured generation, and agentic workflows. To operate under bounded memory, a broad range of KV cache management strategies have been proposed, typically retaining or evicting past states based on attention-derived relevance or simple retention rules~\cite{zhang2023h2o,li2024snapkv,DBLP:conf/iclr/XiaoTCHL24}.

Our analysis reveals that repetition loops are driven by a collapse in attention dynamics: a subset of attention heads repeatedly locks onto a narrow suffix of the generated history, forming highly stable and concentrated attention patterns that trap generation in a low diversity attractor (Figure~\ref{fig:motivation}a)~\cite{liu2024lost,hiraoka2025repetition}. Once such attention collapse occurs, inference-time KV cache reuse preserves and reinforces the same trajectory across subsequent decoding steps. Moreover, because many KV cache management strategies rely on attention-based importance signals, the collapsed attention can become self-reinforcing, repeatedly preserving the same repetitive states and reducing cache diversity, which further amplifies repetition loops (Figure~\ref{fig:motivation}b). We observe this locking and amplification effect to be more pronounced in smaller LLMs, indicating a vulnerability that depends on model scale, making loops easier to trigger and harder to escape. Compounding this issue, existing evaluations typically prioritize downstream task scores under healthy contexts, leaving the regime of degeneration insufficiently controlled and hard to reproduce. This methodological gap makes it difficult to rigorously compare cache policies when reliability fails, motivating the need for a dedicated benchmark and loop-oriented metrics for systematic analysis.

To overcome the lack of reproducible evaluation and the vulnerability of existing cache policies, we introduce LoopBench, a benchmark designed to isolate loop dynamics under controlled conditions. It incorporates explicit triggers covering both schema-constrained generation and recursive instruction following. Complementing this, we introduce loop-oriented metrics to quantify repetition severity and generation instability beyond standard accuracy scores. Building on these findings, we propose LoopGuard, a lightweight KV cache intervention that monitors online degeneration signals. Upon detecting loop onset, LoopGuard breaks the self-reinforcing feedback cycle by pruning repetitive tail spans with a debounced trigger, while preserving anchor tokens and sparsely retaining distant context (Figure~\ref{fig:motivation}c). Experiments demonstrate that LoopGuard acts as a minimally invasive safety net, reducing loop incidence by over 90 percentage points while restoring output diversity and reducing token waste under a fixed cache budget.

\noindent Our contributions are summarized as follows:
\begin{itemize}[leftmargin=0.5em, label=\textbullet, itemsep=0.1ex, topsep=0.1ex]
  \item We identify persistent repetition loops as an attention-level failure mode with head locking, and show that KV cache feedback and attention-guided heuristics can jointly amplify repetition by reducing cache diversity.
  \item We construct LoopBench, a controlled and reproducible benchmark with loop-inducing conditions, and introduce loop-specific metrics that quantify repetition severity and generation instability.
  \item We propose LoopGuard, a lightweight KV cache intervention that detects loop onset online and mitigates repetition by disrupting self-reinforcing cache feedback through targeted pruning of repetitive tail spans under a fixed cache budget.

\end{itemize}

\section{Related Work}
\subsection{Repetitive Generation Mechanisms}
Repetitive generation refers to a degeneration phenomenon in which language models become trapped in producing identical or highly similar token sequences over extended decoding steps~\cite{olsson2022context}. It has been observed across models of varying scales and can severely compromise generation quality and reliability~\cite{xu2022learning, press2021train}. Beyond being a surface-level decoding issue, recent studies suggest that repetition loops reflect a more persistent failure in inference-time dynamics: once the model enters a low-entropy repetitive regime, the generation trajectory can become self-reinforcing and difficult to escape, leading to persistent loops~\cite{li2025loopllm}. Repetition has also been investigated as an adversarially inducible failure mode, where non-halting prompts or automated prompt optimization reliably drive models into sustained repetition under controlled settings~\cite{hammouri2025non}. While this line of work establishes the reproducibility and stability of repetition loops, it primarily focuses on how to elicit such behaviors, leaving the underlying mechanisms of loop persistence during inference less understood.

\subsection{Efficient Long-context Inference with KV Cache}
As context length increases, managing the KV cache during inference becomes a central challenge for long-context generation. A prominent line of work focuses on selectively retaining or evicting KV entries under a fixed memory budget. Methods such as H$_2$O~\cite{zhang2023h2o} and SnapKV~\cite{li2024snapkv} estimate token importance based on historical attention and retain only a subset of tokens for subsequent attention computation. Recent works such as StreamingLLM~\cite{DBLP:conf/iclr/XiaoTCHL24} and LM-Infinite~\cite{han2024lm} retain a small set of initial tokens in addition to a sliding context window, helping preserve model performance on long sequences. However, these methods rely on predefined retention rules and cannot adapt to evolving inference dynamics. To reduce irreversible information loss, hierarchical or retrieval-based KV cache management has been proposed: InfLLM~\cite{xiao2024infllm} maintains block-level memory units with on-demand retrieval (optionally offloading cold blocks to CPU), and QUEST~\cite{DBLP:conf/icml/TangZZXKH24} introduces query-aware sparsity at page granularity. Overall, many KV cache strategies rely on attention-derived relevance signals; under repetition loops, these signals can become self-reinforcing and biased toward a narrow history, amplifying loops and hindering recovery.

\section{Background and Motivation}
\subsection{KV Cache Formulation}

During inference, KV cache stores previously computed attention keys and values for reuse across decoding steps, enabling efficient long-context generation by reducing redundant computation. For each past token $x_i$, the model computes key and value vectors $(k_i^h, v_i^h)$ for each attention head $h$, and stores them in the cache so that future steps can reuse them without recomputation.

At step $t{+}1$, head $h$ forms a query $q_{t+1}^h$ and attends to a subset of cached positions. 
Let $S_t \subseteq \{0,\dots,t\}$ denote the indices of tokens whose keys and values are currently accessible under a cache policy (the full cache corresponds to $S_t=\{0,\dots,t\}$), which can be viewed as an additional restriction on top of the causal attention mask. 
The attention weights and head output are
\begin{equation}
\label{eq:attention_mechanism}
\begin{aligned}
\alpha_{t+1,i}^h
&= \mathrm{softmax}_{i \in S_t}\!\left(
\frac{\langle q_{t+1}^h, k_i^h \rangle}{\sqrt{d_h}}
\right), \\
o_{t+1}^h
&= \sum_{i \in S_t} \alpha_{t+1,i}^h \, v_i^h .
\end{aligned}
\end{equation}
The multi-head outputs are then combined to compute the next-token distribution. 
Under memory limits, cache management dynamically updates $S_t$ (typically enforcing $|S_t|\le B$ for a fixed budget $B$), and thus directly constrains which past information can influence decoding at each step.

\begin{figure}[t]
    \centering
    \includegraphics[width=\linewidth]{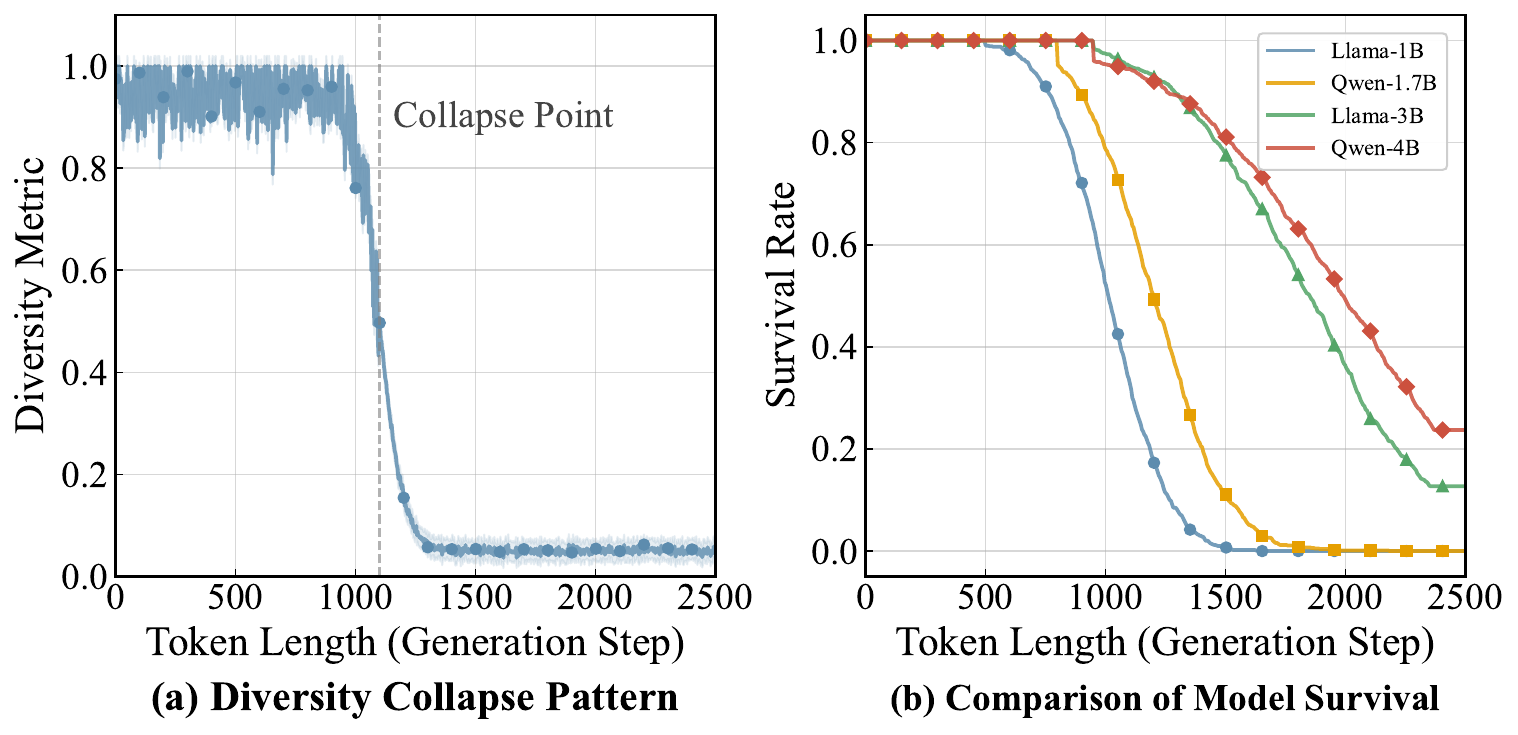}
    \caption{Diversity Collapse and Scale-dependent Survival under Repetition Loops. (a) Diversity metric over decoding steps with a clear collapse point. (b) Survival rate (non-collapsed runs) across model scales under loop-inducing prompts.}
    \label{fig:collapse_survival}
\end{figure}

\subsection{Analysis of Repetitive Generation}
In this paper, we focus on \emph{persistent repetition loops} in \emph{long-context} decoding, where generation collapses into self-sustaining repetition and fails to recover diversity. 
This regime is practically important because long-horizon decoding magnifies failure cost (token waste and latency) and is where KV cache reuse and budgeted cache policies can induce strong feedback effects.

Figure~\ref{fig:collapse_survival} summarizes two key observations. 
In Figure~\ref{fig:collapse_survival}(a), our diversity metric exhibits a two-stage trajectory: it remains high early on, then drops sharply at a collapse point and stays persistently low, after which decoding is trapped in a stable repetitive regime. 
This behavior is consistent with a self-reinforcing autoregressive process in which predictions concentrate once a short span starts to repeat~\cite{li2025loopllm}. 
In Figure~\ref{fig:collapse_survival}(b), we observe a clear scale effect: defining survival at step $t$ as the fraction of runs that have not collapsed by $t$, larger models survive longer while smaller models fail earlier and more frequently. 
These observations motivate online monitoring and targeted disruption once looping begins.

\subsection{Repetition Collapse under RoPE and KV Cache}

We study why repetition loops, once formed, can remain stable under rotary position embeddings (RoPE)~\cite{su2024roformer} with KV cache reuse, i.e., autoregressive decoding where RoPE is applied to the query and key vectors, while past keys and values are cached and reused across steps.

Under RoPE, attention scores admit a block-wise relative-offset form
\begin{equation}
s_{t,i}
\;=\;
\sum_{j=1}^{d_h/2}
(q_t^{(j)})^\top R\!\big(\omega_j (t-i)\big)\, k_i^{(j)},
\label{eq:rope_block}
\end{equation}
where $R(\cdot)$ is a $2{\times}2$ rotation and $\omega_j$ are fixed frequencies.
Thus, scoring depends on indices primarily through the relative offset $(t-i)$.

After collapse, decoding often enters a short-period tail repetition regime with period $P$, where projections over the active suffix become approximately periodic: $q_t \approx q_{t-P}$ and $k_i \approx k_{i-P}$.
Substituting into Eq.~\eqref{eq:rope_block} gives
\begin{equation}
\begin{aligned}
s_{t,i} - s_{t-P,i-P}
\;\approx\;
\sum_{j}
(q_{t-P}^{(j)})^\top
R\!\big(\omega_j (t-i)\big)
k_{i-P}^{(j)}
\\
-\;
\sum_{j}
(q_{t-P}^{(j)})^\top
R\!\big(\omega_j (t-P-(i-P))\big)
k_{i-P}^{(j)}
\;=\; 0.
\end{aligned}
\label{eq:rope_equivariance}
\end{equation}
indicating approximate invariance of scores along the repetition trajectory.

With KV reuse, repeated suffix states are appended while earlier entries remain fixed.
Once attention concentrates on a short recent span, the head output is dominated by a tail operator
\begin{equation}
\mathcal{T}_{\mathrm{tail}}:
\{v_i\}_{i\in\mathcal{B}_t}
\;\mapsto\;
\sum_{i\in\mathcal{B}_t} \alpha_{t,i}\, v_i,
\label{eq:tail_operator}
\end{equation}
which becomes nearly constant across steps under Eq.~\eqref{eq:rope_equivariance}.
Repeatedly appending near-duplicate tail states therefore yields a self-consistent contraction on the tail subspace, stabilizing a low-diversity attractor.
Further details (including error bounds and a stability sketch) are deferred to Appendix~\ref{sec:app_rope_kv}.

\section{Benchmark and Metrics}

\subsection{Loop-inducing Benchmarks}
Our goal is to evaluate \emph{persistent repetition loops} under controlled and reproducible conditions, rather than downstream task performance~\cite{du2025context}. 
To this end, we construct loop-inducing prompts with explicit output constraints and long-horizon decoding, which expose self-reinforcing dynamics while minimizing confounders such as semantic ambiguity. 
Construction details and representative outputs are provided in
Appendix~\ref{sec:app_loopbench_data} and Appendix~\ref{sec:loop_example}.

\textbf{LoopBench-DC} tests schema-constrained structured generation under long, noisy contexts. 
We repurpose the MultiNews subset of LongBench~\cite{bai2024longbench} by replacing summarization with strict JSON-based entity extraction and normalization, where loops appear as duplicated fields or repeated JSON fragments after partial schema satisfaction.

\textbf{LoopBench-RI} targets recursive instruction following over long narratives. 
We repurpose the NarrativeQA subset of LongBench~\cite{bai2024longbench} by replacing QA with an explicit self-referential recursion protocol (draft $\rightarrow$ self-correct $\rightarrow$ expand $\rightarrow$ repeat), seeded with a short logic-stream history; loops manifest as template-like repetition in self-corrections and expansions after diversity collapse.

Despite different surface formats, both benchmarks share the same failure pattern: abrupt diversity collapse followed by a persistent repetitive regime driven by KV cache feedback.

\subsection{Evaluation Protocol}
We evaluate repetition loops under a fixed decoding protocol to isolate cache-coupled dynamics and ensure reproducibility. Unless otherwise specified, we use greedy (deterministic) decoding and cap generation at a maximum of $T_{\max}=2500$ tokens; all models and KV cache strategies use identical prompts and stopping criteria. Each prompt is run three times, yielding identical outputs under deterministic decoding. 

\begin{figure*}[t]
    \centering
    \includegraphics[width=\linewidth]{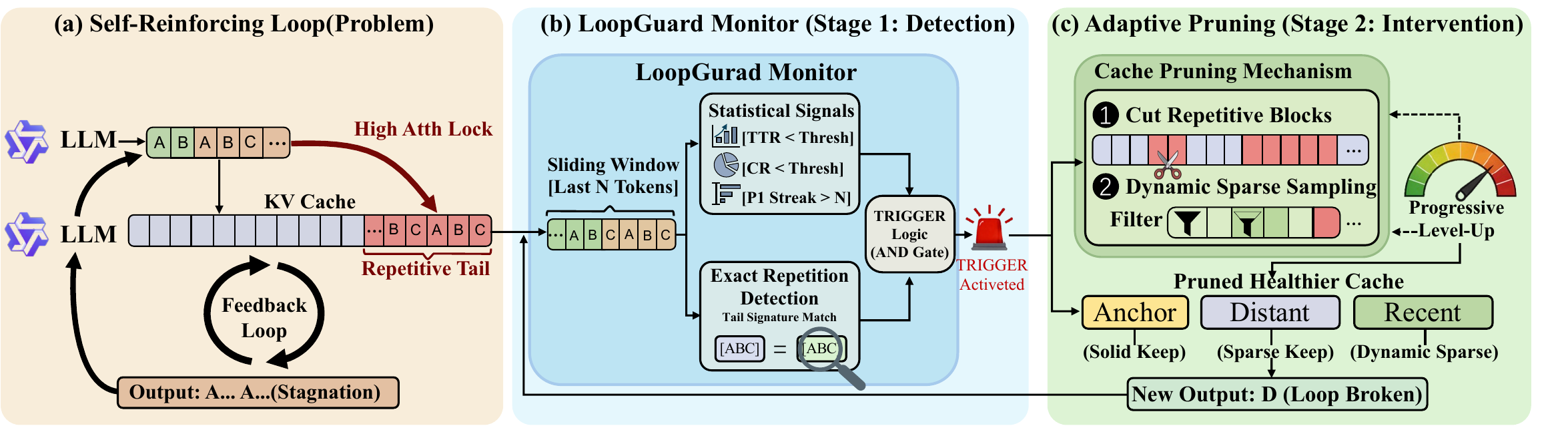}
    \caption{Overview of LoopGuard. (a) Repetition loops driven by attention locking on a narrow tail and stabilized by KV-cache feedback. (b) Online detection via debounced statistical and tail-repetition signals. (c) Adaptive KV cache pruning to break the loop and restore diversity.}
    \label{fig:pipeline}
\end{figure*}

\subsection{Metrics and Collapse Definition}
We detect repetition loops using task-agnostic signals computed from the generated output, capturing diversity collapse, global redundancy, and persistence~\cite{tevet2021evaluating}.

\textbf{Type--Token Ratio (TTR).}
TTR measures lexical diversity as the fraction of unique tokens in the generated sequence~\cite{kettunen2014can}:
\begin{equation}
\mathrm{TTR} = \frac{|\mathrm{uniq}(y_{1:T})|}{T},
\end{equation}
where $y_{1:T}$ denotes the generated token sequence with length $T$.

\textbf{Compression Ratio (CR).}
CR measures global repetition via lossless compressibility~\cite{shaib2024standardizing}:
\begin{equation}
\mathrm{CR} = \frac{|\mathrm{comp}(y_{1:T})|}{|y_{1:T}|},
\end{equation}
where $\mathrm{comp}(\cdot)$ is a standard lossless compressor; lower CR indicates stronger repetition.

\textbf{Generation length.}
We report the number of generated tokens $T$ per run as an auxiliary cost measure. 
In our setting, persistence is reflected by near-budget generation (i.e., failing to terminate early under a fixed $T_{\max}$).

\textbf{Loop detection rule.}
A run is flagged as a repetition loop if it exhibits both abnormally low diversity and high global redundancy, and persists close to the maximum token budget:
\begin{equation}
\label{eq:loop_rule}
\begin{aligned}
\mathrm{Loop}=1 \iff\;&
\mathrm{TTR} \le \theta_{\mathrm{TTR}}
\ \land\ 
\mathrm{CR} \le \theta_{\mathrm{CR}}
\\
&\land\ 
T \ge \theta_{\mathrm{len}} .
\end{aligned}
\end{equation}
where $\theta_{\mathrm{len}}=2480$ (with $T_{\max}=2500$), and $\theta_{\mathrm{TTR}}$ and $\theta_{\mathrm{CR}}$ are fixed thresholds shared across methods, with details reported in Appendix~\ref{sec:appendix_settings}.

\section{The LoopGuard Method}
\subsection{Overview of LoopGuard}

As shown in Figure~\ref{fig:pipeline}, LoopGuard is a lightweight inference-time guard that mitigates self-reinforcing repetition loops via event-driven KV cache interventions. 
Figure~\ref{fig:pipeline}a illustrates the collapse: attention locks onto a repetitive tail and KV cache reuse stabilizes this trajectory, leading to sustained stagnation. 
LoopGuard operates in two stages. 
In Stage~1 (Figure~\ref{fig:pipeline}b), an online monitor tracks recent outputs with a sliding window and triggers only when loop patterns persist, using a debounced decision rule. 
In Stage~2 (Figure~\ref{fig:pipeline}c), LoopGuard prunes the KV cache by removing the repetitive tail and reassembling the cache from anchor tokens, sparsely retained long-range context, and a budgeted recent subset to break the feedback cycle and restore diversity.

\subsection{Online Detection with Debounced Trigger}

LoopGuard models loop onset as an \emph{online change-point} in decoding dynamics, reflected by a rapid shift toward a low-diversity attractor and persistent self-reinforcement.
We therefore build a selective monitor that fuses complementary evidence of degeneration and triggers only under consistent, persistent signals.

Let $\mathcal{W}_t=\{\max(1,t-W+1),\ldots,t\}$ be a sliding window over the partial output $y_{1:t}$.
We compute windowed collapse indicators for diversity (TTR), redundancy (CR), and confidence collapse based on the top-1 probability
$p_{\max}(t)=\max_v \Pr(y_{t+1}=v\mid y_{1:t})$.
We aggregate them via a robust $K$-of-$3$ vote:
\begin{equation}
\label{eq:warn_vote}
\begin{aligned}
\mathrm{Warn}(t)
=\;&
\mathbf{1}[m_{\mathrm{ttr}}(t)<\theta_{\mathrm{ttr}}]
+\mathbf{1}[m_{\mathrm{cr}}(t)<\theta_{\mathrm{cr}}] \\
&+\mathbf{1}[m_{\mathrm{p1}}(t)\ge \theta_{\mathrm{streak}}]
\;\ge\; K .
\end{aligned}
\end{equation}
where $m_{\mathrm{p1}}(t)$ counts consecutive steps with $p_{\max}(t)>\theta_{\mathrm{conf}}$.

To avoid reacting to transient repetitions, we require persistence evidence using (i) a novelty-based stagnation gate over $\mathcal{W}_t$ and (ii) a debounced suffix self-alignment test for short-period tail repetitions; see Appendix~\ref{sec:appendix_settings} for details.
An intervention is issued only after a warmup phase and outside a cooldown window:
\begin{equation}
\label{eq:trigger_final}
\begin{split}
\mathrm{Trigger}(t)
=\;
\mathbf{1}[t\ge t_{\min}]
\;\land\;
\mathbf{1}[\mathrm{cd}(t)=0] \\
\land\;
\mathrm{Warn}(t)
\;\land\;
\big(\mathrm{Stall}(t)\lor \mathrm{TailRep}(t)\big).
\end{split}
\end{equation}
where $\mathrm{cd}(t)$ is a cooldown counter reset after each intervention.
This debounced design keeps the monitor dormant during healthy decoding, while responding promptly once the degenerative regime becomes stable.

\subsection{Adaptive KV Cache Pruning}
When $\mathrm{Trigger}(t)$ fires, LoopGuard performs an event-driven cache reset that removes the corrupted tail trajectory while retaining enough context for coherent continuation under a fixed budget $B$.
Concretely, we construct a retained index set $\mathcal{I}_{\mathrm{keep}}\subseteq\{0,\ldots,t\}$ and prune all KV entries outside it, with $|\mathcal{I}_{\mathrm{keep}}|\le B$:
\begin{equation}
\label{eq:ikeep}
\mathcal{I}_{\mathrm{keep}}
=
\mathcal{I}_{\mathrm{anchor}}
\;\cup\;
\mathcal{I}_{\mathrm{sparse}}
\;\cup\;
\mathcal{I}_{\mathrm{recent}} .
\end{equation}

\noindent \textbf{Anchors.}
We preserve a short prefix
$\mathcal{I}_{\mathrm{anchor}}=\{0,\ldots,N_{\mathrm{anchor}}-1\}$
to keep the instruction and global constraints stable after pruning.

\noindent \textbf{Sparse long-range context.}
Let $\mathcal{W}^{\mathrm{recent}}_t=\{t-R+1,\ldots,t\}$ denote the recent window.
From the remaining history $\{0,\ldots,t\}\setminus \mathcal{W}^{\mathrm{recent}}_t$, we form $\mathcal{I}_{\mathrm{sparse}}$ by uniform subsampling under a fixed budget.
This preserves coarse long-range coverage while avoiding dense contiguous spans that can reinforce local repetition.

\noindent \textbf{Tail-cleaned recent context.}
Let $\mathcal{I}_{\mathrm{bad}}(t)\subseteq \mathcal{W}^{\mathrm{recent}}_t$ denote indices belonging to detected tail-repetition blocks.
We exclude them from the recent candidates,
\begin{equation}
\mathcal{C}_{\mathrm{recent}}
=
\mathcal{W}^{\mathrm{recent}}_t
\setminus
\mathcal{I}_{\mathrm{bad}}(t),
\end{equation}
and select $\mathcal{I}_{\mathrm{recent}}$ as a budgeted subsample of $\mathcal{C}_{\mathrm{recent}}$.
This explicitly cuts the loop-inducing suffix while retaining a small amount of local context for smooth continuation.

\noindent \textbf{Progressive aggressiveness.}
LoopGuard maintains a discrete aggressiveness level $\ell$ that controls the recent budget $B_{\mathrm{recent}}(\ell)$.
If degeneration re-triggers shortly after an intervention, we increase $\ell$ (up to a cap) and shrink $B_{\mathrm{recent}}(\ell)$ accordingly, reducing reliance on the immediate past and shifting attention toward anchors and sparse long-range context.
Finally, we map $\mathcal{I}_{\mathrm{keep}}$ to cache slots and apply index selection to prune keys and values across layers, after which decoding resumes under the cooldown constraint.

\subsection{Complexity Analysis}

LoopGuard incurs negligible overhead on top of standard autoregressive decoding.
The online monitor maintains windowed statistics over the last $W$ tokens: TTR and novelty updates cost $O(W)$ per step in a straightforward implementation, while the compression-based redundancy score is evaluated every $s$ steps (a fixed interval), yielding an amortized $O(W/s)$ overhead.
Cache pruning is event-driven and executes only when $\mathrm{Trigger}(t)$ fires.
It applies an index-gather to keys and values across $L$ layers with cost
$O\!\left(L\,|\mathcal{I}_{\mathrm{keep}}|\,d_{\mathrm{kv}}\right)$,
which is linear in the retained cache length and requires no additional forward passes.
In practice, since $|\mathcal{I}_{\mathrm{keep}}|\le B$ is bounded and interventions are rare, the added cost is small compared to the base attention computation.

\section{Experiments}

\begin{table*}[t]
\centering
\small
\setlength{\tabcolsep}{6pt}
\begin{tabular}{l l l c c c c}
\toprule
Dataset & Model & Method &  CR$\uparrow$ & TTR$\uparrow$ & Avg Len$\downarrow$ & Loop Rate$\downarrow$\\
\midrule

\multirow{8}{*}{LoopBench-DC}
& \multirow{4}{*}{Qwen3-1.7B}
& Full Cache              & 0.064 & 0.086 & 2495 & 100\% \\
& & StreamingLLM           & 0.085 & 0.118 & 2496 & 100\% \\
& & H$_2$O                 & 0.052 & 0.079 & 2495 & 100\% \\
& & \textbf{LoopGuard (Ours)} & \textbf{0.219} & \textbf{0.528} & \textbf{1326} & \textbf{1.3\%} \\
\cmidrule(lr){2-7}
& \multirow{4}{*}{Llama3.2-1B}
& Full Cache              & 0.044 & 0.065 & 2498 & 100\% \\
& & StreamingLLM           & 0.058 & 0.087 & 2498 & 100\% \\
& & H$_2$O                 & 0.039 & 0.054 & 2496 & 100\% \\
& & \textbf{LoopGuard (Ours)} & \textbf{0.207} & \textbf{0.495} & \textbf{1382} & \textbf{1.7\%} \\

\midrule

\multirow{8}{*}{LoopBench-RI}
& \multirow{4}{*}{Qwen3-1.7B}
& Full Cache              & 0.105 & 0.097 & 2379 & 93.5\% \\
& & StreamingLLM           & 0.097 & 0.122 & 2418 & 97.2\% \\
& & H$_2$O                 & 0.080 & 0.095 & 2486 & 99.8\% \\
& & \textbf{LoopGuard (Ours)} & \textbf{0.261} & \textbf{0.474} & \textbf{1412} & \textbf{2.7\%} \\
\cmidrule(lr){2-7}
& \multirow{4}{*}{Llama3.2-1B}
& Full Cache              & 0.089 & 0.119 & 2479 & 94.7\% \\
& & StreamingLLM           & 0.081 & 0.097 & 2485 & 98.4\% \\
& & H$_2$O                 & 0.068 & 0.084 & 2496 & 100.0\% \\
& & \textbf{LoopGuard (Ours)} & \textbf{0.305} & \textbf{0.489} & \textbf{1457} & \textbf{2.3\%} \\

\bottomrule
\end{tabular}
\caption{Main results on LoopBench under a fixed KV budget of 1024 tokens and a shared maximum generation budget of $T_{\max}=2500$.
We report loop rate and loop-oriented metrics (CR/TTR) together with average generation length (Avg Len) as an auxiliary token-cost measure.
Best results are in \textbf{bold}.}
\label{tab:main_loopbench}
\end{table*}

\subsection{Experiment Setting}
\noindent \textbf{Datasets.} We evaluate our method on two distinct categories of tasks: (1) LoopBench, our newly constructed benchmark comprising \textit{LoopBench-DC} (Data Constraints) and \textit{LoopBench-RI} (Recursive Instructions), which are specifically designed to induce and evaluate self-reinforcing repetition loops; and (2) 2WikiMultihopQA from LongBench~\cite{bai2024longbench}, which serves as a standard long-context baseline to verify that our method preserves multi-hop reasoning capabilities under strict memory budgets.

\noindent \textbf{Models.} We conduct experiments using the Qwen~\cite{yang2025qwen3} and Llama~\cite{dubey2024llama} model families, with a specific focus on smaller scales that are typically more sensitive to cache dynamics.

\noindent \textbf{Baselines.} We compare LoopGuard against three representative KV cache strategies: (1) Full Cache (standard attention with no eviction); (2) H$_2$O~\cite{zhang2023h2o}, an attention-based heavy-hitter eviction policy; and (3) StreamingLLM~\cite{DBLP:conf/iclr/XiaoTCHL24}, a position-based sliding window policy with attention sinks.

\noindent \textbf{Implementation Details.} To ensure a fair comparison, all efficient KV cache methods are constrained to a fixed KV cache budget of 1024 tokens. For LoopGuard, we set the online monitor window $W=256$, the debounce threshold $K=2$, and reserve a fixed anchor size of 32 tokens. We utilize greedy decoding for LoopBench to strictly isolate deterministic loop dynamics, and cap generation with a maximum output length of $T_{\max}=2500$ tokens. All experiments are conducted on NVIDIA A100 GPUs.

\subsection{Main Results on LoopBench}
Table~\ref{tab:main_loopbench} reports the main results on LoopBench under a fixed KV cache budget and a shared maximum generation budget of $T_{\max}=2500$. Across both LoopBench-DC and LoopBench-RI, and for both Qwen3-1.7B and Llama3.2-1B, LoopGuard consistently and substantially reduces persistent repetition compared to all baselines. On LoopBench-DC, all baseline methods (Full Cache, StreamingLLM, and H$_2$O) collapse in all runs, reaching a loop rate of $100\%$ for both model families, whereas LoopGuard reduces the loop rate to $1.3\%$ (Qwen3-1.7B) and $1.7\%$ (Llama3.2-1B). On LoopBench-RI, baselines still fail in the vast majority of cases ($93.5\%$--$100\%$), while LoopGuard lowers loop rates to $2.3\%$--$2.7\%$, demonstrating loop suppression across distinct loop-inducing regimes.

Beyond loop frequency, LoopGuard also restores diversity and reduces token waste. Baseline outputs are highly compressible and lexically repetitive, with very low CR ($\le 0.105$) and TTR ($\le 0.122$), and they almost always exhaust the maximum token budget with average lengths near $T_{\max}$. In contrast, LoopGuard substantially increases CR to about $0.21$--$0.22$ and TTR to $0.47$--$0.53$, indicating recovery of lexical and structural diversity, while reducing average generation length to roughly $1.3$k--$1.45$k tokens; this reduction reflects natural termination after diversity is restored rather than aggressive early stopping. Among the baselines, H$_2$O consistently exhibits the most severe degeneration (lowest CR/TTR and highest loop rates), aligning with our analysis that attention-guided eviction can amplify self-reinforcing repetition once attention collapses onto a narrow history span, whereas StreamingLLM alters retention behavior but does not fundamentally prevent loop formation.

\subsection{Generalization to Long-Context QA}
\begin{table}[t]
\centering
\small
\setlength{\tabcolsep}{7pt}
\begin{tabular}{l c}
\toprule
Method & F1$\uparrow$ \\
\midrule
Full Cache & 50.49 \\
StreamingLLM & 21.05 \\
H$_2$O & 33.73 \\
\textbf{LoopGuard (Ours)} & \textbf{53.61} \\
\bottomrule
\end{tabular}
\caption{Performance on 2WikiMultiHopQA (LongBench) measured by F1 under the same KV budget setting as LoopBench.}
\label{tab:2wiki_f1}
\end{table}

Beyond loop-inducing settings, we evaluate whether LoopGuard preserves general long-context capabilities on 2WikiMultiHopQA from LongBench~\cite{bai2024longbench}. Table~\ref{tab:2wiki_f1} reports the F1 scores. LoopGuard achieves the best performance (53.61 F1), slightly improving over Full Cache (50.49), indicating that our event-driven pruning does not harm multi-hop reasoning and can even help by reducing distraction from redundant or noisy context. In contrast, StreamingLLM and H$_2$O exhibit substantial drops (21.05 and 33.73 F1, respectively), suggesting that fixed sliding-window retention or attention-guided eviction may discard globally relevant evidence required for multi-hop aggregation. Overall, these results demonstrate that LoopGuard is not only effective at suppressing persistent repetition loops, but also maintains strong task performance under strict KV budgets, supporting its plug-in reliability intervention design.

\begin{figure}[t]
    \centering
    \includegraphics[width=\linewidth]{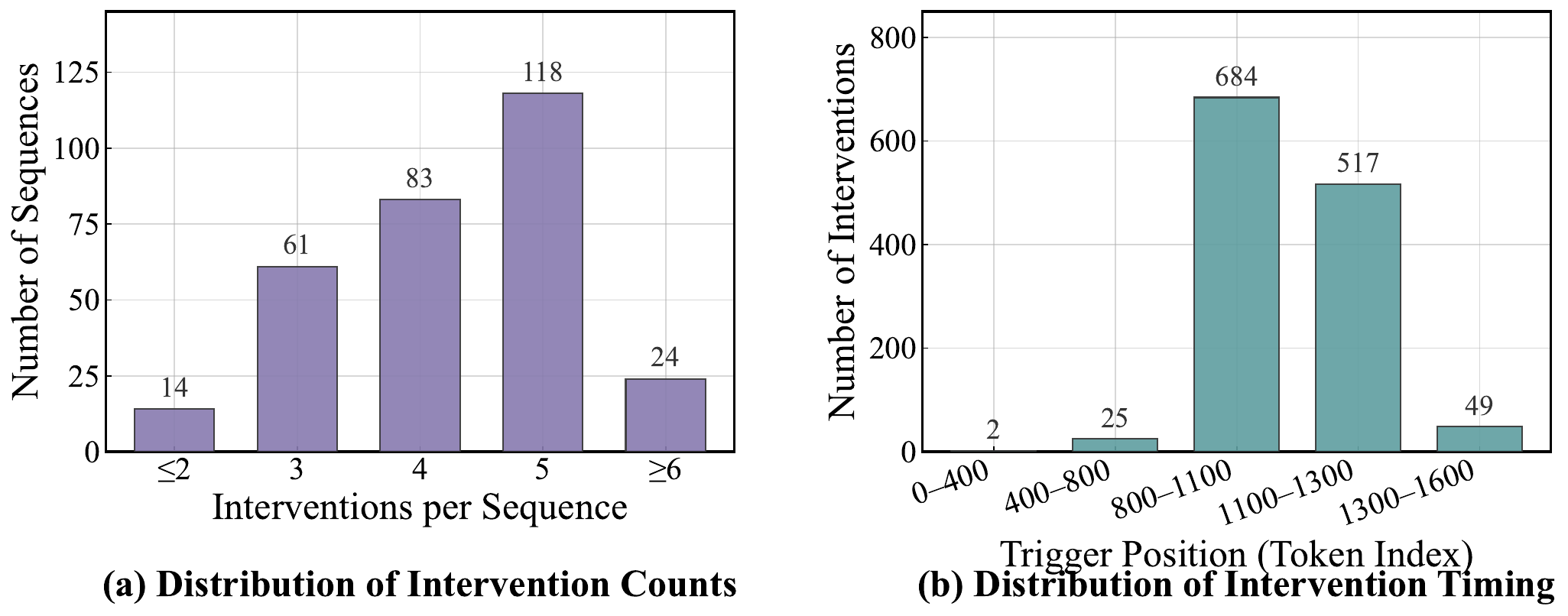}
    \caption{LoopGuard intervention behavior under loop-inducing prompts.
  (a) Distribution of the number of interventions per sequence.
  (b) Distribution of intervention trigger positions.}
  \label{fig:loopguard-intervention}
\end{figure}

\subsection{Intervention Frequency and Timing}

Figure~\ref{fig:loopguard-intervention} summarizes LoopGuard's intervention behavior. In Figure~\ref{fig:loopguard-intervention}a, interventions per sequence concentrate in the range of 3--5, with 5 being the most common case; sequences with very few ($\le$2) or many ($\ge$6) interventions are relatively rare. Figure~\ref{fig:loopguard-intervention}b shows that triggers are overwhelmingly concentrated at 800--1300 tokens (accounting for the vast majority of interventions), while early-stage triggers (0--400) are almost absent. Overall, LoopGuard is mostly dormant early and activates primarily around the typical collapse window, consistent with an event-driven reliability guard rather than an always on pruning policy.

\subsection{Ablation Study of LoopGuard}

Table~\ref{tab:ablation_loopguard} summarizes the ablation results of LoopGuard.
Without any guard, all sequences collapse into loops, which confirms that the prompts induce a highly stable degenerate regime under the fixed KV budget.
Always-on pruning offers only a limited reduction in loop rate (79.8\%) even though it intervenes very frequently (12.5 times per sequence on average).
More importantly, over half of its triggers happen early (57.6\%), suggesting that it prunes the cache before clear degeneration emerges, which disrupts normal generation yet fails to break the self reinforcing loop once it forms.
This highlights that aggressive pruning alone is not sufficient; the intervention must be timed and targeted.

The single signal variant that relies on the top 1 confidence streak ($p_1$) improves loop suppression substantially, lowering the loop rate to 35.4\%.
However, it still requires 7.6 interventions per sequence and shows a noticeable fraction of early triggers (21.2\%).
This pattern indicates that a single confidence cue is informative but not decisive: high confidence can arise from benign formatting or locally predictable spans, so triggering on $p_1$ alone tends to intervene too often and sometimes too early.
In contrast, the full LoopGuard achieves the lowest loop rate (1.3\%) with fewer interventions (4.8 on average) and almost no early triggers (1.7\%).
These results support the full design choice that combines multiple lightweight signals with a debounced trigger, so interventions are issued mainly after persistent degeneration is corroborated, which improves precision and reduces unnecessary cache modifications.

\begin{table}[t]
\centering
\small
\setlength{\tabcolsep}{1.5pt}  
\begin{tabular}{l c c c}
\toprule
Variant 
& Loop Rate (\%)$\downarrow$ 
& Avg. \#Int.$\downarrow$ 
& Early (\%)$\downarrow$ \\
\midrule
No Guard & 100.0 & 0.00 & -- \\
Always-on Pruning & 79.8 & 12.5 & 57.6 \\
Single-signal ($p_1$) & 35.4 & 7.6 & 21.2 \\
\textbf{Full LoopGuard} & \textbf{1.3} & \textbf{4.8} & \textbf{1.7} \\
\bottomrule
\end{tabular}
\caption{Ablation study of LoopGuard variants on LoopBench.
We report loop rate, average number of interventions per sequence (Avg.~\#Int.), and the fraction of early triggers.
Lower is better for all metrics.}
\label{tab:ablation_loopguard}
\end{table}

\section{Conclusion}

In this study, we investigated persistent repetition loops in long-context decoding and linked them to an attention-level collapse where a subset of heads locks onto a narrow suffix, which KV cache reuse then reinforces via feedback. We introduced LoopBench with loop-oriented metrics beyond downstream task scores, and proposed LoopGuard, an event-driven KV cache intervention that detects loop onset and breaks the feedback cycle by pruning repetitive tails while retaining anchors and sparse long-range context under a fixed budget. Experiments show that LoopGuard cuts loop incidence by over 90 percentage points, restores output diversity, and preserves performance on standard long-context QA.

\section*{Limitations}

This work focuses on inference-time mitigation of persistent repetition loops through KV cache interventions, and does not address degeneration rooted in training objectives or model architectures. While LoopGuard is effective at disrupting stabilized loops once collapse occurs, it is not designed to prevent all forms of low-diversity generation, such as semantically repetitive but lexically varied outputs that may not trigger our detection signals. In addition, our evaluation centers on controlled loop-inducing scenarios to enable reproducible analysis of reliability failures under long-horizon decoding. Although LoopBench isolates dynamics coupled with the KV cache, it does not cover all failure modes that may arise in open-ended or interactive generation settings. Extending reliability evaluation to broader regimes and combining inference-time guards with complementary training-time approaches remain important directions for future study.





\bibliography{custom}

\appendix

\section{Additional Theoretical Analysis for RoPE KV Cache}
\label{sec:app_rope_kv}

\subsection{Explicit RoPE Formulation}
We briefly recall the explicit construction of rotary position embeddings (RoPE)~\cite{su2024roformer}
and show how it yields the block-wise relative-offset form used in the main text.
Let $d_h$ be the per-head dimension and assume $d_h$ is even.
Each query and key vector is partitioned into $d_h/2$ two-dimensional blocks,
\(
q_t = (q_t^{(1)}, \dots, q_t^{(d_h/2)}),
\;
k_i = (k_i^{(1)}, \dots, k_i^{(d_h/2)}).
\)
For block $j$, RoPE applies a position-dependent rotation
\begin{equation}
\tilde q_t^{(j)} = R(\omega_j t)\, q_t^{(j)},
\qquad
\tilde k_i^{(j)} = R(\omega_j i)\, k_i^{(j)},
\end{equation}
where
\(
R(\theta) =
\begin{pmatrix}
\cos\theta & -\sin\theta \\
\sin\theta & \cos\theta
\end{pmatrix}
\)
and $\omega_j$ is the frequency associated with block $j$.
The attention score is then
\begin{equation}
s_{t,i}
= \sum_{j=1}^{d_h/2}
(\tilde q_t^{(j)})^\top \tilde k_i^{(j)}
= \sum_{j=1}^{d_h/2}
(q_t^{(j)})^\top R\!\big(\omega_j (t-i)\big)\, k_i^{(j)},
\end{equation}
which recovers Eq.~(\ref{eq:rope_block}) in the main text.
Notably, the score depends on positions only through the relative offset $(t-i)$.

\subsection{Approximate Score Invariance and Error Bounds}
We formalize the approximation in Eq.~(\ref{eq:rope_equivariance}) by bounding the score difference under periodic tail repetition.
Let $\mathcal{B}_t$ denote the (short) tail window on which attention concentrates after collapse.
Assume there exists a period $P$ such that, for the relevant tail region,
\begin{equation}
\label{eq:approx_periodic_assump}
\begin{aligned}
\|q_t - q_{t-P}\|_2 &\le \delta_q, \\
\|k_i - k_{i-P}\|_2 &\le \delta_k, \quad \forall\, i \in \mathcal{B}_t .
\end{aligned}
\end{equation}

for small $\delta_q,\delta_k$, where $\|\cdot\|_2$ denotes the Euclidean norm.

\begin{lemma}[Approximate shift invariance]
\label{lem:approx_shift_invariance}
Under assumption~\eqref{eq:approx_periodic_assump}, the RoPE attention scores satisfy
\begin{equation}
\label{eq:score_diff_bound}
\begin{aligned}
& \big| s_{t,i} - s_{t-P,i-P} \big|
\le \sum_{j=1}^{d_h/2}
\Big(
\|k_{i-P}^{(j)}\|_2 \delta_q
\\
& \quad  
+
\|q_{t-P}^{(j)}\|_2 \delta_k
+
\delta_q \delta_k
\Big),
\quad \forall\, i \in \mathcal{B}_t .
\end{aligned}
\end{equation}

\end{lemma}

\begin{proof}
By the RoPE score form, and since $(t-i)=(t-P)-(i-P)$, the rotation term
$R\!\big(\omega_j(t-i)\big)$ is identical in both scores.
For each block $j$, write
\begin{equation}
\begin{aligned}
& (q_t^{(j)})^\top R\!\big(\omega_j(t-i)\big) k_i^{(j)}
-
(q_{t-P}^{(j)})^\top R\!\big(\omega_j(t-i)\big) k_{i-P}^{(j)}
\\
& \quad =
\big(q_t^{(j)}-q_{t-P}^{(j)}\big)^\top R(\cdot)\, k_i^{(j)}
\\
& \qquad +
(q_{t-P}^{(j)})^\top R(\cdot)\, \big(k_i^{(j)}-k_{i-P}^{(j)}\big).
\end{aligned}
\end{equation}
Taking absolute values and applying triangle inequality and Cauchy--Schwarz,
together with $\|R(\cdot)u\|_2=\|u\|_2$ (orthogonality of $R$), yields
\begin{equation}
\begin{aligned}
& \Big|
(q_t^{(j)}-q_{t-P}^{(j)})^\top R(\cdot)\, k_i^{(j)}
\Big|
\le
\|q_t^{(j)}-q_{t-P}^{(j)}\|_2\,\|k_i^{(j)}\|_2,
\\
& \Big|
(q_{t-P}^{(j)})^\top R(\cdot)\, (k_i^{(j)}-k_{i-P}^{(j)})
\Big|
\le
\|q_{t-P}^{(j)}\|_2\,\|k_i^{(j)}-k_{i-P}^{(j)}\|_2.
\end{aligned}
\end{equation}
Using $\|k_i^{(j)}\|_2 \le \|k_{i-P}^{(j)}\|_2 + \|k_i^{(j)}-k_{i-P}^{(j)}\|_2$
and the bounds in~\eqref{eq:approx_periodic_assump}, and summing over $j$, gives~\eqref{eq:score_diff_bound}.
\end{proof}

When $\delta_q$ and $\delta_k$ are small, the score difference is negligible, justifying the approximate invariance used in the main text.
Moreover, softmax is a smooth mapping; under bounded logits, small perturbations in scores induce bounded changes in attention weights, so the attention distribution over $\mathcal{B}_t$ varies slowly along the repetition trajectory.

\subsection{Stability Sketch of the Tail Operator}
We provide a sketch explaining why the tail operator in Eq.~(\ref{eq:tail_operator}) can stabilize repetition under KV reuse.
Consider the cache-coupled decoding dynamics
\begin{equation}
\begin{aligned}
h_{t+1} &= F(h_t, \mathrm{KV}_t), \\
\mathrm{KV}_{t+1} &= \mathrm{Append}(\mathrm{KV}_t, k_{t+1}, v_{t+1}),
\end{aligned}
\end{equation}
and suppose that, after collapse, attention mass concentrates on a short tail index set $\mathcal{B}_t$.

Under approximate score invariance (Lemma~\ref{lem:approx_shift_invariance}), the attention weights over $\mathcal{B}_t$ vary slowly across steps.
Consequently, the induced tail operator
\(
\mathcal{T}_{\mathrm{tail}}(\{v_i\}_{i\in\mathcal{B}_t})
=
\sum_{i\in\mathcal{B}_t}\alpha_{t,i}\,v_i
\)
remains approximately time-invariant along the repetition trajectory.
Repeatedly appending near-duplicate tail states therefore induces a self-consistent update restricted to a low-dimensional ``tail'' subspace.

Linearizing the cache-coupled dynamics around a periodic repetition trajectory yields a local recurrence for perturbations.
If the linearized update is locally contractive on directions orthogonal to the tail subspace (e.g., the corresponding Jacobian has spectral radius $<1$ on those directions), then such perturbations decay, resulting in a stable low-diversity attractor.
Breaking this attractor thus requires an explicit perturbation to the KV cache, motivating cache-level interventions.

\section{Implementation Details and Hyperparameter Settings}
\label{sec:appendix_settings}

\subsection{Threshold Settings for Loop Detection}
\label{sec:appendix_thresholds}

In Eq.~\eqref{eq:loop_rule}, we use fixed thresholds for the diversity and redundancy metrics to identify persistent repetition loops.
Specifically, we set the Type--Token Ratio threshold to $\theta_{\mathrm{TTR}} = 0.2$ and the Compression Ratio threshold to $\theta_{\mathrm{CR}} = 0.12$ across all benchmarks and models.

The choice of $\theta_{\mathrm{TTR}} = 0.2$ reflects a strong collapse in lexical diversity: fewer than 20\% of generated tokens are unique, which rarely occurs in healthy long-context generation but is typical once repetition loops emerge.
Similarly, $\theta_{\mathrm{CR}} = 0.12$ corresponds to a high degree of global redundancy under lossless compression, indicating extensive repetition beyond short-range patterns.

We observe consistent separation between normal generation and persistent looping behavior across LoopBench-DC and LoopBench-RI.
Qualitatively similar trends hold under moderate threshold variations, suggesting that loop detection is not sensitive to precise tuning.

\subsection{Online Detection: Parameters and Criteria}

This section summarizes the hyperparameters and criteria used by LoopGuard for online loop detection.

\noindent \textbf{Windowed indicators.} All statistics are computed over a sliding window of $W=256$ tokens.
We monitor lexical diversity (Type--Token Ratio) and redundancy (Compression Ratio),
with thresholds $\theta_{\mathrm{ttr}}=0.2$ and $\theta_{\mathrm{cr}}=0.12$, respectively.
These values consistently separate healthy decoding from collapsed repetition across models and prompts.

\noindent \textbf{Confidence collapse.} We track the top-1 prediction probability
$p_{\max}(t)=\max_v \Pr(y_{t+1}=v\mid y_{1:t})$
and maintain a streak counter that increments when $p_{\max}(t)>\theta_{\mathrm{conf}}=0.9$.
A confidence-collapse signal is raised when the streak exceeds $\theta_{\mathrm{streak}}=6$ steps.

\noindent \textbf{Persistence gates.} To suppress transient repetition, LoopGuard additionally requires persistence evidence.
We use (i) a novelty-based stagnation test with threshold $\theta_{\mathrm{nov}}=0.02$ and
$\theta_{\mathrm{stall}}=4$ consecutive steps, and
(ii) a debounced suffix repetition check that detects short-period tail alignment.
Details of the suffix-matching procedure are omitted for brevity.

\noindent \textbf {Triggering constraints.} An intervention is permitted only after a warmup of $t_{\min}=64$ decoding steps
and is followed by a cooldown window of $32$ steps.
All hyperparameters are shared across models and benchmarks without per-instance tuning.

\section{LoopBench Construction Details}
\label{sec:app_loopbench_data}

This appendix details how we construct \textbf{LoopBench}, a controlled set of loop-inducing prompts designed to elicit persistent repetition collapse under long-horizon decoding.
LoopBench contains two subsets with distinct surface forms but a shared failure pattern: an abrupt diversity collapse followed by a stable repetitive regime.

\subsection{LoopBench-DC: Data-Constraint Structured Generation}
\label{sec:app_loopbench_dc}

\textbf{Goal.} LoopBench-DC targets schema-constrained structured generation under long and noisy contexts, where loops typically manifest as duplicated fields or repeated JSON fragments after partial schema satisfaction.

\textbf{Prompt format.} Each instance is a single long prompt consisting of:
(i) a fixed \emph{instruction header} requesting entity extraction and normalization (e.g., mapping entity mentions to Wikipedia titles),
(ii) an explicit \emph{output schema constraint} specifying the required JSON item fields \{\texttt{name}, \texttt{wikipedia url}, \texttt{type}\} and the allowed type set \{\texttt{Organization}, \texttt{Person}, \texttt{Location}, \texttt{Event}, \texttt{Other}\},
and (iii) a \emph{dirty context} paragraph containing long-form transcripts with substantial redundancy and repeated spans.
This design makes the task semantically unambiguous (extract-and-list) while enforcing strict formatting, which exposes loop persistence once the model begins to repeat suffix fragments.

\textbf{Source and size.} We repurpose the MultiNews subset of LongBench by replacing summarization with the above strict JSON-based entity extraction prompt.
LoopBench-DC contains \textbf{300} prompts, with an average prompt length of \textbf{3,853} tokens (tokenized on the full prompt text, excluding generated outputs).

\subsection{LoopBench-RI: Infinite Recursive Interpretation}
\label{sec:app_loopbench_ri}

\textbf{Goal.} LoopBench-RI targets recursive instruction following over long narratives, where loops typically appear as template-like repetitions in self-corrections and expansions.

\textbf{Prompt format.} Each instance contains:
(i) a long \emph{source narrative} (book-style or story-like text),
(ii) a strict \emph{recursive protocol} that forces the model to iterate \texttt{Draft Interpretation} $\rightarrow$ \texttt{Self-Correction} $\rightarrow$ \texttt{Recursive Expansion} $\rightarrow$ \texttt{Repeat}, and
(iii) a short \emph{logic-stream seed history} that provides 1--2 completed iterations (draft + self-correction) and then truncates at the next \texttt{Draft} header.
By construction, the model is required to continue the loop immediately, which encourages persistent generation and makes repetition collapse highly reproducible under deterministic decoding.

\textbf{Source and size.} We repurpose the NarrativeQA subset of LongBench by replacing question answering with the above self-referential recursion protocol.
LoopBench-RI contains \textbf{200} prompts, with an average prompt length of \textbf{5,376} tokens (tokenized on the full prompt text, excluding generated outputs).

\subsection{Reproducibility Notes}
\label{sec:app_loopbench_repro}

Across both subsets, prompts are designed to minimize semantic ambiguity while imposing strong structural constraints (schema or recursion), so that once a repetition trajectory is entered, it becomes stable and persistent under long-horizon decoding.
Representative prompt templates and qualitative looping outputs are provided separately in Appendix~\ref{sec:loop_example}.

\section{Representative Looping Outputs}
\label{sec:loop_example}
This appendix lists representative model outputs under loop-inducing prompts in LoopBench-DC.
We include examples from different KV cache strategies, where baseline methods may exhibit duplicated entries and sustained tail repetition once the schema is partially satisfied.
These qualitative cases complement the main quantitative results by illustrating typical failure patterns in practice.

\section{Attention Visualization under Repetition Loops}
\label{sec:appendix_attention}

This appendix provides attention visualizations for representative loop-inducing generations.
Figure~\ref{fig:appendix_attention_barcode} shows attention maps from selected heads across multiple layers
during long-context decoding.
Across layers and heads, attention concentrates into narrow, vertically aligned stripes,
indicating persistent focus on a small repetitive suffix of the generated history.
These barcode-like patterns qualitatively support our analysis that repetition loops
are driven by attention-level collapse and head locking.

\clearpage
\onecolumn
\begin{tcolorbox}[
    breakable,                 
    enhanced,
    colback=white,             
    colframe=black,            
    sharp corners,             
    boxrule=0.8pt,             
    title=\textbf{Input (Full Prompt)}, 
    coltitle=black,
    colbacktitle=gray!20,      
    attach boxed title to top left={yshift=-\tcboxedtitleheight},
    boxed title style={frame hidden, size=small, sharp corners},
    top=1.5em                  
]
    \small 
    \textbf{Context:} \textbf{Extract a list of the people, organizations, locations and other entities in the following text, disambiguating and normalizing each name to the title of its corresponding Wikipedia article. Classify the Type of each as one of the following: "Organization, Person, Location, Event, Other". Output the results in JSON format: }\textit{(Context truncated)} tonight, several developing stories as we come on the air. news coming in, the suspected chinese spy balloon hovering over the u.s. where it is now, and will there be a window to shoot it down? also tonight, the dangerous life-threatening cold moving into the northeast. the live readings already. first, the pentagon tracking that chinese balloon, the intelligence bay hanging beneath it, the size of three busses. where it was spotted today over the u.s., where it's believed to be headed now, 60,000 feet in the air. will the u.s. shoot this down? secretary of state antony blinken postponing his high stakes trip to china. mola lenghi, mary bruce standing by at the white house. the deadly cold already tonight, the national weather service is calling it a once in a generation arctic blast. 25 million americans, multiple states. new york city, philadelphia, boston, wind-chill readings in maine expected to reach 60 below zero. and where the wind chill is already 106 degrees below zero tonight. \dots so where is it tonight, and where u.s. authorities believe it's headed. among the first sightings over billings, montana. that intelligence bay hanging beneath the balloon, that's the size of three busses. the pentagon tonight on discussions over whether or not to shoot it down, and will there be a small window of opportunity to do just that? what we're learning tonight. secretary of state antony blinken postponing his high-stakes trip to beijing, calling china's actions unacceptable. so, will this be shot down? if so, where? and what china is saying tonight. mary bruce at the white house, and abc's mola lenghi leading us off tonight in montana. reporter: tonight, the pentagon confirming the massive chinese spy balloon is on the move 60,000 feet above the ground and heading east. the balloon continues t move eastward and is currently over the center of the continental united states. what the heck is that? reporter: a senior u.s. official tells abc news the balloon now appears headed towards north carolina. across the country today, americans with their eyes on the skies, posting images like this one, the balloon floating over missouri. commercial pilots radioing in. we got that balloon in sight also. looks like it's way up there, maybe 50,000 feet or so. reporter: the balloon is huge with a technology bay attached below that is itself the size of three busses, loaded with high resolution cameras, according to a senior u.s. official, equipped with what appear to be solar panels on the side that could power its technology. it first entered american airspace over alaska, then flying into southwest canada before dipping down over billings, montana. this thing is up in the sky. what the heck is that? that thing is not the moon. any help would be appreciated. reporter: montana republican congressman ryan zinke was blunt, tweeting ``shoot. it. down.'' the pentagon today sayig that option was considered and rejected, for now. we assessed that currently it does not pose a physical or military risk to people on the ground. for now we are continuing to monitor and review options. reporter: the white house today saying president biden agreed with the pentagon's strong recommendation. the risks involved with shooting down the balloon just too great. any potential debris field would be significant and potentially cause civilian injuries or deaths or significant property damage. reporter: for its part, china insisting today the balloon is not a spy vessel at all but a civilian airship used for research, mainly meteorological purposes. the pentagon flatly rejecting that claim. we are aware of the prc's statement. however, the fact is, we know that it's a surveillance pballo. we have clearly communicated that this balloon is violating u.s. air space and international law and that this is unacceptable. reporter: china says the balloon deviated far from its planned course, but the pentagon questioning that as well, saying the balloon can be maneuvered. a senior u.s. official telling abc news, it's likely being navigated with spy satellites. this all happening as secretary of state antony blinken prepared for a weekend trip to beijing.
\end{tcolorbox}

\newcommand{\dup}[1]{%
  \begingroup
  \sethlcolor{yellow!25}%

  \soulregister\_7
  \soulregister\{7
  \soulregister\}7
  \soulregister\textbf7
  \soulregister\dots7

  \def\_{\char95\allowbreak}%

  \makeatletter
  \def\SOUL@uleveryspace##1{%
    ##1\SOUL@ulcolor{\SOUL@ulleaders}\hskip\spaceskip
  }%
  \makeatother

  \hl{#1}%
  \endgroup
}

\newcommand{\loopwarn}[1]{%
  \begingroup
  \setlength{\fboxsep}{1.2pt}%
  \fcolorbox{red!85!black}{white}{\strut\textbf{[#1]}}%
  \endgroup
}

\begin{tcolorbox}[
    breakable,
    enhanced,
    colback=white,
    colframe=black,
    sharp corners,
    boxrule=0.8pt,
    title=\textbf{Output (Full Cache)},
    coltitle=black,
    colbacktitle=gray!20,
    attach boxed title to top left={yshift=-\tcboxedtitleheight},
    boxed title style={frame hidden, size=small, sharp corners},
    top=1.5em
]
    \ttfamily\small
    \setlength{\parindent}{0pt}\setlength{\parskip}{0pt}
    \noindent \{"name":"Pentagon","wikipedia url":"\dots","type":"Organization"\},\par
    \noindent \{"name":"U.S. Department of State","wikipedia url":"\dots","type":"Organization"\},\par
    \noindent \{"name":"U.S. Air Force","wikipedia url":"\dots","type":"Organization"\},\par
    \noindent \{"name":"U.S. Army","wikipedia url":"\dots","type":"Organization"\},\par
    \noindent \{"name":"U.S. Navy","wikipedia url":"\dots","type":"Organization"\},\par
    \noindent \{"name":"U.S. Coast Guard","wikipedia url":"\dots","type":"Organization"\},\par
    \noindent \{"name":"U.S. Marine Corps","wikipedia url":"\dots","type":"Organization"\},\par
    \noindent \{"name":"U.S. Army Corps of Engineers","wikipedia url":"\dots","type":"Organization"\},\par
    \noindent \{"name":"U.S. National Weather Service","wikipedia url":"\dots","type":"Organization"\},\par
    \noindent \{"name":"U.S. Department of Homeland Security","wikipedia url":"\dots","type":"Organization"\},\par
    \noindent \{"name":"U.S. Department of Justice","wikipedia url":"\dots","type":"Organization"\},\par
    \noindent \{"name":"U.S. Department of the Interior","wikipedia url":"\dots","type":"Organization"\},\par
    \noindent \{"name":"U.S. Department of the Treasury","wikipedia url":"\dots","type":"Organization"\},\par
    \noindent \{"name":"U.S. Department of Education","wikipedia url":"\dots","type":"Organization"\},\par
    \noindent \{"name":"U.S. Department of Veterans Affairs","wikipedia url":"\dots","type":"Organization"\},\par
    \noindent \{"name":"U.S. Department of Agriculture","wikipedia url":"\dots","type":"Organization"\},\par
    \noindent \{"name":"U.S. Department of Energy","wikipedia url":"\dots","type":"Organization"\},\par
    \noindent \{"name":"U.S. Department of Commerce","wikipedia url":"\dots","type":"Organization"\},\par
    \noindent \{"name":"U.S. Department of Justice","wikipedia url":"\dots","type":"Organization"\},\par
    \noindent \{"name":"U.S. Department of Veterans Affairs","wikipedia url":"\dots","type":"Organization"\},\par
    \noindent \dots\par
    \noindent \dup{\{"name":"U.S. Air Force Academy","wikipedia url":"\dots","type":"Organization"\},}\par
    \noindent \dup{\{"name":"U.S. Air Force Academy","wikipedia url":"\dots","type":"Organization"\},}\par
    \noindent \dup{\{"name":"U.S. Air Force Academy","wikipedia url":"\dots","type":"Organization"\},}\par
    \noindent \dup{\{"name":"U.S. Air Force Academy","wikipedia url":"\dots","type":"Organization"\},}\par
    \noindent \dup{\{"name":"U.S. Air Force Academy","wikipedia url":"\dots","type":"Organization"\},}\par
    \noindent \dup{\{"name":"U.S. Air Force Academy","wikipedia url":"\dots","type":"Organization"\},}\par
    \noindent \dup{\{"name":"U.S. Air Force Academy","wikipedia url":"\dots","type":"Organization"\},}\par
    \noindent \dup{\{"name":"U.S. Air Force Academy","wikipedia url":"\dots","type":"Organization"\},}\par
    \noindent \dup{\{"name":"U.S. Air Force Academy","wikipedia url":"\dots","type":"Organization"\},}\par
    \noindent \dup{\{"name":"U.S. Air Force Academy","wikipedia url":"\dots","type":"Organization"\},}\par
    \noindent \dup{\dots}\par
    \noindent \loopwarn{Model enters infinite loop}\par
\end{tcolorbox}

\begin{tcolorbox}[
    breakable,
    enhanced,
    colback=white,
    colframe=black,
    sharp corners,
    boxrule=0.8pt,
    title=\textbf{Output (StreamingLLM)},
    coltitle=black,
    colbacktitle=gray!20,
    attach boxed title to top left={yshift=-\tcboxedtitleheight},
    boxed title style={frame hidden, size=small, sharp corners},
    top=1.5em
]
    \ttfamily\small
    \setlength{\parindent}{0pt}\setlength{\parskip}{0pt}
    \noindent \{"name":"Pentagon","wikipedia url":"\dots","type":"Organization"\},\par
    \noindent \{"name":"U.S. Department of State","wikipedia url":"\dots","type":"Organization"\},\par
    \noindent \{"name":"China","wikipedia url":"\dots","type":"Organization"\},\par
    \noindent \{"name":"Mary Bruce","wikipedia url":"\dots","type":"Person"\},\par
    \noindent \{"name":"Mola Lenghi","wikipedia url":"\dots","type":"Person"\},\par
    \noindent \{"name":"Rob Marciano","wikipedia url":"\dots","type":"Person"\},\par
    \noindent \{"name":"Victor Oquendo","wikipedia url":"\dots","type":"Person"\},\par
    \noindent \{"name":"Eva Pilgrim","wikipedia url":"\dots","type":"Person"\},\par
    \noindent \{"name":"Robin Roberts","wikipedia url":"\dots","type":"Person"\},\par
    \noindent \{"name":"Alex Murdaugh","wikipedia url":"\dots","type":"Person"\},\par
    \noindent \{"name":"National Weather Service","wikipedia url":"\dots","type":"Organization"\},\par
    \noindent \{"name":"United States","wikipedia url":"\dots","type":"Location"\},\par
    \noindent \{"name":"Missouri","wikipedia url":"\dots","type":"Location"\},\par
    \noindent \{"name":"Montana","wikipedia url":"\dots","type":"Location"\},\par
    \noindent \{"name":"North Carolina","wikipedia url":"\dots","type":"Location"\},\par
    \noindent \{"name":"Alaska","wikipedia url":"\dots","type":"Location"\},\par
    \noindent \{"name":"Billings, Montana","wikipedia url":"\dots","type":"Location"\},\par
    \noindent \{"name":"New York City","wikipedia url":"\dots","type":"Location"\},\par
    \noindent \{"name":"Philadelphia","wikipedia url":"\dots","type":"Location"\},\par
    \noindent \{"name":"Boston","wikipedia url":"\dots","type":"Location"\},\par
    \noindent \{"name":"Maine","wikipedia url":"\dots","type":"Location"\},\par
    \noindent \{"name":"Wind chill","wikipedia url":"\dots","type":"Other"\},\par
    \noindent \{"name":"Arctic blast","wikipedia url":"\dots","type":"Other"\},\par
    \noindent \{"name":"Unemployment rate","wikipedia url":"\dots","type":"Other"\},\par
    \noindent \{"name":"Federal Reserve Bank","wikipedia url":"\dots","type":"Other"\},\par
    \noindent \{"name":"Weather balloon","wikipedia url":"\dots","type":"Other"\},\par
    \noindent \{"name":"Surveillance balloon","wikipedia url":"\dots","type":"Other"\},\par
    \noindent \dots\par
    \noindent \dup{\{"name":"Federal Reserve","wikipedia url":"\dots","type":"Other"\},}\par
    \noindent \dup{\{"name":"Jobs report","wikipedia url":"\dots","type":"Other"\},}\par
    \noindent \dup{\{"name":"Unemployment rate","wikipedia url":"\dots","type":"Other"\},}\par
    \noindent \dup{\{"name":"Federal Reserve Bank","wikipedia url":"\dots","type":"Other"\},}\par
    \noindent \dup{\{"name":"Jobs report","wikipedia url":"\dots","type":"Other"\},}\par
    \noindent \dup{\dots}\par
    \noindent \loopwarn{Model enters infinite loop}\par
\end{tcolorbox}

\begin{tcolorbox}[
    breakable,
    enhanced,
    colback=white,
    colframe=black,
    sharp corners,
    boxrule=0.8pt,
    title=\textbf{Output (H2o)},
    coltitle=black,
    colbacktitle=gray!20,
    attach boxed title to top left={yshift=-\tcboxedtitleheight},
    boxed title style={frame hidden, size=small, sharp corners},
    top=1.5em
]
    \ttfamily\small
    \setlength{\parindent}{0pt}\setlength{\parskip}{0pt}
    \noindent \{"name":"Pentagon","wikipedia url":"\dots","type":"Organization"\},\par
    \noindent \{"name":"U.S. Department of State","wikipedia url":"\dots","type":"Organization"\},\par
    \noindent \{"name":"U.S. Air Force","wikipedia url":"\dots","type":"Organization"\},\par
    \noindent \{"name":"U.S. Army","wikipedia url":"\dots","type":"Organization"\},\par
    \noindent \{"name":"U.S. Navy","wikipedia url":"\dots","type":"Organization"\},\par
    \noindent \{"name":"U.S. Coast Guard","wikipedia url":"\dots","type":"Organization"\},\par
    \noindent \{"name":"U.S. Marine Corps","wikipedia url":"\dots","type":"Organization"\},\par
    \noindent \{"name":"U.S. Army Corps of Engineers","wikipedia url":"\dots","type":"Organization"\},\par
    \noindent \{"name":"U.S. National Weather Service","wikipedia url":"\dots","type":"Organization"\},\par
    \noindent \{"name":"U.S. Department of Homeland Security","wikipedia url":"\dots","type":"Organization"\},\par
    \noindent \dots\par
    \noindent \dup{\{"name":"U.S. Department of Justice","wikipedia url":"\dots","type":"Organization"\},}\par
    \noindent \dup{\{"name":"U.S. Department of the Interior","wikipedia url":"\dots","type":"Organization"\},}\par
    \noindent \dup{\{"name":"U.S. Department of the Treasury","wikipedia url":"\dots","type":"Organization"\},}\par
    \noindent \dup{\{"name":"U.S. Department of Education","wikipedia url":"\dots","type":"Organization"\},}\par
    \noindent \dup{\{"name":"U.S. Department of Veterans Affairs","wikipedia url":"\dots","type":"Organization"\},}\par
    \noindent \dup{\{"name":"U.S. Department of Agriculture","wikipedia url":"\dots","type":"Organization"\},}\par
    \noindent \dup{\{"name":"U.S. Department of Energy","wikipedia url":"\dots","type":"Organization"\},}\par
    \noindent \dup{\{"name":"U.S. Department of Commerce","wikipedia url":"\dots","type":"Organization"\},}\par
    \noindent \dup{\{"name":"U.S. Department of Justice","wikipedia url":"\dots","type":"Organization"\},}\par
    \noindent \dup{\{"name":"U.S. Department of the Interior","wikipedia url":"\dots","type":"Organization"\},}\par
    \noindent \dup{\{"name":"U.S. Department of the Treasury","wikipedia url":"\dots","type":"Organization"\},}\par
    \noindent \dup{\{"name":"U.S. Department of Veterans Affairs","wikipedia url":"\dots","type":"Organization"\},}\par
    \noindent \dup{\{"name":"U.S. Department of Agriculture","wikipedia url":"\dots","type":"Organization"\},}\par
    \noindent \dup{\{"name":"U.S. Department of Energy","wikipedia url":"\dots","type":"Organization"\},}\par
    \noindent \dup{\{"name":"U.S. Department of Commerce","wikipedia url":"\dots","type":"Organization"\},}\par
    \noindent \dup{\{"name":"U.S. Department of the Interior","wikipedia url":"\dots","type":"Organization"\},}\par
    \noindent \dup{\{"name":"U.S. Department of the Treasury","wikipedia url":"\dots","type":"Organization"\},}\par
    \noindent \dup{\{"name":"U.S. Department of Agriculture","wikipedia url":"\dots","type":"Organization"\},}\par
    \noindent \dup{\{"name":"U.S. Department of Energy","wikipedia url":"\dots","type":"Organization"\},}\par
    \noindent \dup{\{"name":"U.S. Department of Commerce","wikipedia url":"\dots","type":"Organization"\},}\par
    \noindent \dup{\{"name":"U.S. Department of the Interior","wikipedia url":"\dots","type":"Organization"\},}\par
    \noindent \dup{\dots}\par
    \noindent \loopwarn{Model enters infinite loop}\par
\end{tcolorbox}

\begin{tcolorbox}[
    breakable,
    enhanced,
    colback=white,
    colframe=black,
    sharp corners,
    boxrule=0.8pt,
    title=\textbf{Output (Our Method)},
    coltitle=black,
    colbacktitle=gray!20,
    attach boxed title to top left={yshift=-\tcboxedtitleheight},
    boxed title style={frame hidden, size=small, sharp corners},
    top=1.5em
]
    \ttfamily\small
    \setlength{\parindent}{0pt}\setlength{\parskip}{0pt}
    \noindent \{"name":"Pentagon","wikipedia url":"\dots","type":"Organization"\},\par
    \noindent \{"name":"U.S. Department of State","wikipedia url":"\dots","type":"Organization"\},\par
    \noindent \{"name":"China","wikipedia url":"\dots","type":"Organization"\},\par
    \noindent \{"name":"Secretary of State Antony Blinken","wikipedia url":"\dots","type":"Person"\},\par
    \noindent \{"name":"Mary Bruce","wikipedia url":"\dots","type":"Person"\},\par
    \noindent \{"name":"Mola Lenghi","wikipedia url":"\dots","type":"Person"\},\par
    \noindent \{"name":"Rob Marciano","wikipedia url":"\dots","type":"Person"\},\par
    \noindent \{"name":"Victor Oquendo","wikipedia url":"\dots","type":"Person"\},\par
    \noindent \{"name":"Eva Pilgrim","wikipedia url":"\dots","type":"Person"\},\par
    \noindent \{"name":"Robin Roberts","wikipedia url":"\dots","type":"Person"\},\par
    \noindent \{"name":"Alex Murdaugh","wikipedia url":"\dots","type":"Person"\},\par
    \noindent \{"name":"National Weather Service","wikipedia url":"\dots","type":"Organization"\},\par
    \noindent \{"name":"United States","wikipedia url":"\dots","type":"Location"\},\par
    \noindent \{"name":"Missouri","wikipedia url":"\dots","type":"Location"\},\par
    \noindent \{"name":"Montana","wikipedia url":"\dots","type":"Location"\},\par
    \noindent \{"name":"North Carolina","wikipedia url":"\dots","type":"Location"\},\par
    \noindent \{"name":"Alaska","wikipedia url":"\dots","type":"Location"\},\par
    \noindent \{"name":"Southwest Canada","wikipedia url":"\dots","type":"Location"\},\par
    \noindent \{"name":"Billings, Montana","wikipedia url":"\dots","type":"Location"\},\par
    \noindent \{"name":"New York City","wikipedia url":"\dots","type":"Location"\},\par
    \noindent \{"name":"Philadelphia","wikipedia url":"\dots","type":"Location"\},\par
    \noindent \{"name":"Boston","wikipedia url":"\dots","type":"Location"\},\par
    \noindent \{"name":"Maine","wikipedia url":"\dots","type":"Location"\},\par
    \noindent \{"name":"Wind chill","wikipedia url":"\dots","type":"Other"\},\par
    \noindent \{"name":"Arctic blast","wikipedia url":"\dots","type":"Other"\},\par
    \noindent \{"name":"Unemployment rate","wikipedia url":"\dots","type":"Other"\},\par
    \noindent \dots\par
    \noindent \{"name":"Federal Reserve Bank","wikipedia url":"\dots","type":"Other"\},\par
    \noindent \{"name":"Jobs report","wikipedia url":"\dots","type":"Other"\},\par
    \noindent \{"name":"Bungee jump","wikipedia url":"\dots","type":"Other"\},\par
    \noindent \{"name":"Debris field","wikipedia url":"\dots","type":"Other"\},\par
    \noindent \{"name":"Spy satellite","wikipedia url":"\dots","type":"Other"\},\par
    \noindent \{"name":"Weather balloon","wikipedia url":"\dots","type":"Other"\},\par
    \noindent \{"name":"Surveillance balloon","wikipedia url":"\dots","type":"Other"\},\par
    \noindent \dots\par
    \noindent \{"name":"United Nations","wikipedia":"\dots","type":"Organization"\}\par
\end{tcolorbox}

\clearpage
\twocolumn

\begin{figure*}[t]
    \centering
    \includegraphics[height=0.82\textheight]{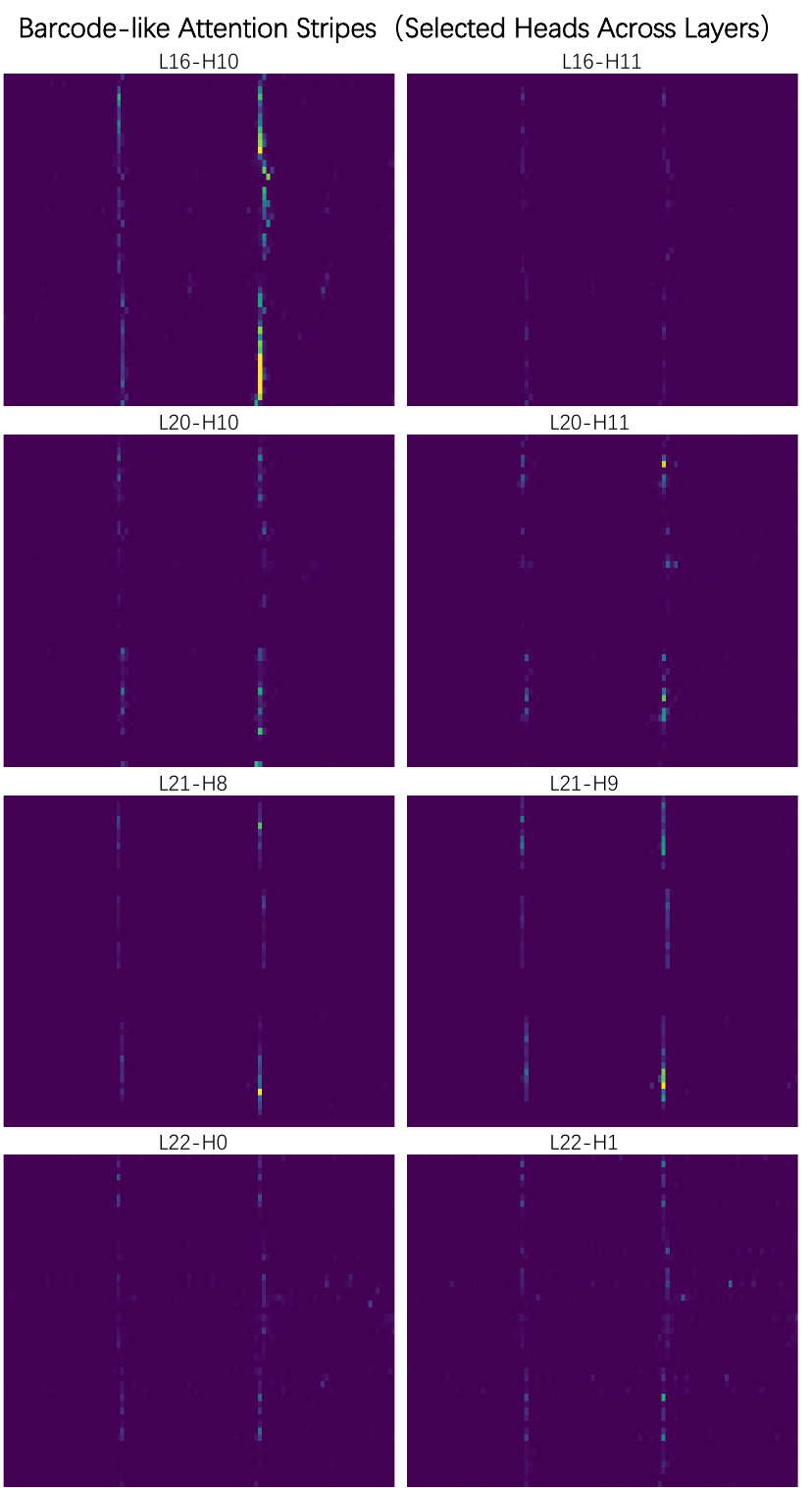}
    \caption{
    \textbf{Barcode-like attention patterns under repetition collapse.}
    Visualization of attention maps from selected heads across multiple layers
    (L16, L20--L22) during persistent repetition loops.
    Each panel exhibits narrow, vertically aligned high-attention stripes,
    indicating head-level locking onto a small repetitive suffix of the history.
    }
    \label{fig:appendix_attention_barcode}
\end{figure*}

\end{document}